\def\year{2018}\relax
\documentclass[letterpaper]{article} 
\usepackage{aaai18}  
\usepackage{times}  
\usepackage{helvet}  
\usepackage{courier}  
\usepackage{url}  
\usepackage{graphicx}  

\usepackage{booktabs}       
\usepackage{amsfonts}       
\usepackage{nicefrac}       
\usepackage{microtype}      

\usepackage{wrapfig}
\usepackage{subfigure} 
\usepackage{array}
\usepackage{tabulary}
\newcolumntype{K}[1]{>{\centering\arraybackslash}p{#1}}
\usepackage{hhline}
\usepackage{algorithm}
\usepackage{algorithmic}
\usepackage{bm}
\usepackage{amsmath}
\usepackage{amssymb}
\usepackage{amsthm}
\usepackage{enumitem}
\usepackage[flushleft]{threeparttable}
\usepackage[textsize=tiny]{todonotes}
\newtheorem{theorem}{Theorem}
\newtheorem{corollary}{Corollary}
\newcommand{\Pt}{P_\text{trg}}
\newcommand{\Ps}{P_\text{src}}

\def\argmin{\operatornamewithlimits{argmin}}

\title{Kernel Robust Bias-Aware Prediction under Covariate Shift}
\author{Anqi Liu, Rizal Fathony, Brian D. Ziebart \\ University of Illinois at Chicago}

\begin{document}


%
%
%
%
%
%
%
%
%

\maketitle

\begin{abstract}
Under covariate shift, training (source) data 
and testing (target) data 
differ in input space distribution, but share the same conditional label 
distribution. This poses a challenging machine learning task.
Robust Bias-Aware 
(RBA)
prediction provides the conditional label distribution
that is robust to the worst-case logarithmic loss for the target distribution
while matching feature expectation constraints from the source distribution.
However,  
employing RBA with 
insufficient feature constraints may result in high certainty predictions for much of the source
data, while leaving too much uncertainty for target data predictions. 
To overcome this issue,
we extend the representer theorem to the RBA setting, enabling
minimization of regularized
expected target risk by a reweighted kernel expectation under the source distribution. By applying kernel methods, we establish 
consistency guarantees and demonstrate better performance of the RBA 
classifier than competing methods on 
synthetically biased UCI datasets as well as datasets that have natural covariate shift. 
\end{abstract}

\section{Introduction}
In 
standard 
supervised 
machine learning, data used to evaluate the generalization error of a classifier is assumed to be 
independently 
drawn from the same distribution that 
generates 
training samples. 
This assumption of independent and identically distributed (IID) data 
is partially violated in the covariate shift setting  
\cite{shimodaira2000improving}, where the 
conditional label distribution $P(y|x)$ 
is shared by source and target data, but 
the distribution on input variables $P(x)$
differs 
between source and target samples, i.e., $\Ps(x)$ differs from 
$\Pt(x)$. 
All models trained under IID assumptions can 
suffer from covariate shift 
and provide overly optimistic extrapolation when generalizing to new data \cite{fan2005improved}. 
An intuitive and traditional method to 
address 
covariate shift is by importance weighting  \cite{shimodaira2000improving,zadrozny2004learning}, which
tries to de-bias the objective loss function by weighting each instance with the ratio ${\Pt(x)}/{\Ps(x)}$. However, importance weighting not only results in 
high variance predictions, but 
also 
only provides generalization performance guarantees when  
strong conditions 
are met by the source and target data distributions 
\cite{cortes2010learning}. 


The 
recently developed 
robust bias aware 
(RBA) approach to covariate shift 
 \cite{liu2014robust} is 
based on a minimax robust estimation formulation \cite{grunwald2004game} that 
assumes 
the worst case 
conditional label 
distribution and 
requires 
only source feature
expectation matching as constraints. 
The approach provides 
conservative 
target predictions 
when 
the 
target 
distribution 
does not have sufficient statistical 
support from the source 
data. 
This statistical support is defined by the choice of source statistics or 
features. 
The classifier tries to make the prediction certainty 
under 
the target distribution as small as possible,
but 
feature matching constraints prevent it from doing
so fully. 
As a result, 
less 
restrictive 
feature constraints
produce 
less certain predictions 
on target data from 
the 
resulting 
classifier. 
As shown in Figure~\ref{syn_cmp}(a), 
with limited features, the classifier may 
allocate 
most of the certainty 
under portions of 
the source distribution (solid line) 
where the target distribution (dashed line) density is small 
to satisfy the source feature expectation matching constraints, 
leaving too much uncertainty in portions of
the 
target distribution.  
On the other hand, when there are more 
restrictive features constraining the conditional label distribution,
the classifier 
produces a 
better model 
of 
the data and 
gives 
more informative predictions with less target entropy and logloss, as in Figure~\ref{syn_cmp}(b). 
This relation inspires our 
contribution: 
leveraging 
kernel methods 
to provide 
higher dimensional features to the RBA classifier
without introducing a proportionate computational burden.

According to the representer theorem \cite{Kimeldorf197182}, the minimizer of regularized empirical loss in 
reproducing kernel 
Hilbert space can be represented by 
a 
linear combination of kernel products evaluated on training data. 
Model 
parameters
are then 
obtained by estimating the coefficients 
of this 
linear combination. However, in the robust bias-aware classification framework, the objective function of the dual 
problem is the regularized expected 
logarithmic loss
under the target data distribution.  
It 
cannot 
be computed explicitly using data because 
 labeled 
target 
samples are unavailable. 
Meanwhile, the distribution discrepancy 
when 
evaluating 
the 
risk function and sampling training data prevents us from applying the representer theorem directly. 
{
\begin{figure}[htp]
\begin{tabular}{cc}

(a) First moment features&\, 
(b)  Third moment features\\
\includegraphics[height=3.3cm]{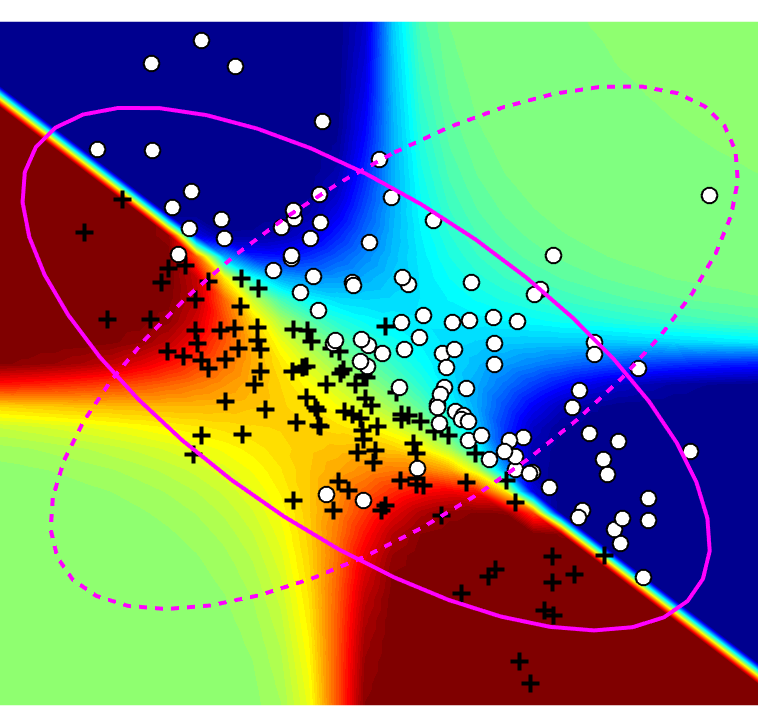} &
 \includegraphics[height=3.3cm]{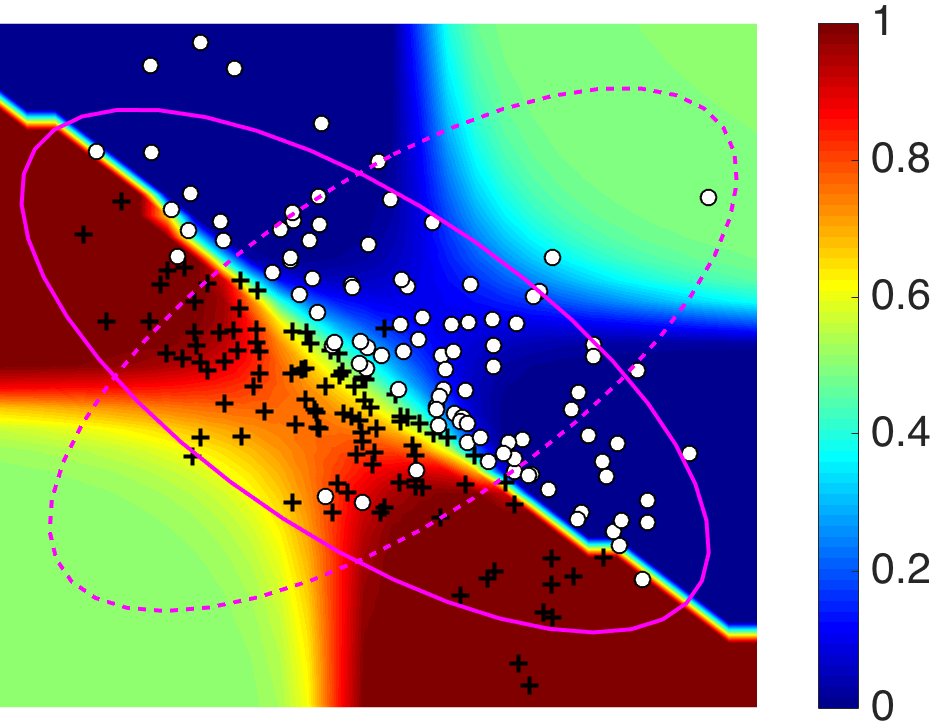}  \\
logloss: 0.74&logloss: 0.53\\
 entropy: 0.93  &  entropy: 0.73
\end{tabular}
\caption{Performance comparison with the robust bias aware classifier using 
first-order features (a) and first-order through third-order features 
(b).  
Labeled source data samples (`o' and `+' classes), source (solid line) and target (dashed line)        
distributions that data are drawn from are shown.  
The colormap represents the predicted probability, $P(y =$`+'$|x)$. The intersection of 
the source distribution and the target distribution is better predicted with 
third-order features and is much more uncertain when only using 
first moment features. The corresponding target logloss and entropy are shown. }
\label{syn_cmp}
\end{figure}
}

A quantitative form of the representer theorem 
has been proposed 
that holds for the continuous case  \cite{de2004some}
in which a minimizer over a distribution---rather than discrete samples---is 
sought. 
The minimizer of regularized expected risk is represented as the expectation under the same probability distribution instead of a linear combination of the 
training 
data. We 
utilize this result 
to extend the representer theorem 
for RBA prediction 
in the covariate shift setting. 
We show that 
the minimizer of 
the 
regularized
expected target risk 
can be represented 
as a reweighted kernel expectation under the source distribution.
This enables us to apply kernel methods to the robust bias aware classifier.

In this paper, we explore the theoretical foundation of kernel methods for 
robust covariate shift prediction. 
We investigate the underlying effect brought by kernelization and establish consistency properties that are realized by applying 
kernel methods to RBA prediction.  
We then demonstrate 
the empirical advantages of the kernel robust bias aware classifier on 
synthetically biased benchmark datasets as well as 
datasets that have natural covariate shift bias.  

\section{Related Work}
%
To address the shifts between training and testing input distributions, which is known as covariate shift, 
existing methods often try to reweight the source samples,
denoted $\tilde{P}_{\text{src}}(x)$, 
to 
make them 
more representative of 
target distribution
samples. 
The theoretical argument 
supporting 
this approach \cite{shimodaira2000improving} 
is that reweighting 
is asymptotically optimal for minimizing target distribution
log loss (equivalently, maximizing target distribution
log-likelihood): 
\begin{align}
&\mathbb{E}_{P_{\text{trg}}(x) P(y|x)}[
\text{loss}(\hat{P}(Y|X), Y)] = \notag \\
&\lim_{n \rightarrow \infty}
\mathbb{E}_{\tilde{P}_{\text{src}}(x) \tilde{P}(y|x)}\left[
\frac{\Pt(X)}{P_{\text{src}}(X)} 
\text{loss}(\hat{P}(Y|X), Y)
\right], \label{eq:reweight} 
\end{align}
where we use $\hat{P}(y|x)$ to represent the estimated predictor and $\tilde{P}(y|x)$ is the empirical distribution. 

Most existing covariate shift research follows this idea of seeking an 
unbiased estimator of target risk
 \cite{sugiyama2005model}.  
Significant attention has been paid to 
estimating the density ratio ${\Pt(x)}/{\Ps(x)}$, 
which strongly impacts predictive performance  \cite{Bickel2009,dudik2005correcting}.  
Direct estimation methods 
estimate the 
density 
ratio by minimizing information theoretical criterion like KL-divergence \cite{sugiyama2008direct,kanamori2009efficient,yamada2011relative} or 
matching kernel means  \cite{huang2006correcting,yu2012analysis}
rather than estimating the ratio's numerator and denominator densities  
separately. 
Other methods \cite{wen2014robust} consider the ratio as an inner parameter within the model and relate
the ratio with model misspecification.
There are also methods for specific models of the covariate shift mechanism \cite{zadrozny2004learning,sugiyama2007covariate}. 
Additional methods have also been recently proposed to 
address some of the limitations of 
importance weighting \cite{reddi2015doubly}.

Theoretical analyses have uncovered the brittleness of importance  
weighting for covariate shift by  
analyzing 
its statistical learning
bounds \cite{ben2007analysis,cortes2008sample}.  
Cortes et al. (\citeyear{cortes2010learning}) establish 
generalization bounds for learning importance 
weighting under covariate shift  
that only hold 
when the second moment of sampled importance weights is bounded, 
$\mathbb{E}_{\Ps(x)}[(\Pt(X)/\Ps(X))^2] < \infty$. 
When not bounded, 
a small number of data points 
with large importance weights
can dominate the reweighted loss, resulting in high variance predictions.

Kernel 
methods 
have been 
employed for estimating the density ratio in 
importance weighting methods;
for example, 
kernel mean matching \cite{huang2006correcting,yu2012analysis}. 
uses the 
core 
idea that the kernel mean in a reproducing
kernel Hilbert space (RKHS) of the source data should be close to that of the reweighted target data and the optimal density ratio is obtained by minimizing 
this 
difference.
Kernel methods have 
also served 
as a bridge between the source and the target domains in broader transfer 
learning or domain adaptation problems.
In these approaches, kernel methods are used to
project source data and target data into a latent space where the distance 
between the two distributions is small or can be minimized \cite{pan2010survey}.

These existing applications of kernel methods for covariate shift 
are orthogonal to our approach because 
they are based on 
empirical risk minimization formulations with the assumption
that source data could somehow be transformed to match target distributions. This differs substantially from our robust approach. 

\section{Approach}

\subsection{Robust bias-aware classifier}
The robust bias-aware classification model is 
based on a minimax robust estimation framework (\ref{eq:game}). 
Under this framework, an 
estimator player $\hat{P}(Y|X)$
first 
chooses 
a conditional label distribution to minimize the logloss and then 
an 
adversarial player $\check{P}(Y|X)$
chooses a 
label distribution from the set ($\Xi$) of statistic-matching 
conditional probability to maximize the logloss  \cite{liu2014robust}: 

{\small
\begin{align}
\min_{\hat{P}(Y|X) } \;
\max_{\check{P}(Y|X) \in \Xi} 
\; \text{loss}_{\Pt(X)}\left(\check{P}(Y|X), \hat{P}(Y|X)\right).
\label{eq:game}
\end{align} 
}%
Under IID settings, it is known that robust loss minimization  
is equivalent and dual to empirical risk minimization 
\cite{grunwald2004game}.  
From this perspective, RBA prediction modifies the dual robust loss 
minimization problem in contrast to existing importance weighting methods, 
which modify the primal empirical risk minimization problem. 
After distinguishing between source and target distributions and using the
logarithmic loss, the robust optimization problem
reduces 
to a maximum entropy problem: 
{\small
\begin{align}
&\!\!\!\max_{\hat{P}(Y|X)}\! 
H_{\Pt(x)} 
(Y|X) 
\triangleq \mathbb{E}_{\Pt(x)\hat{P}(y|x)}[-\!\log \hat{P}(Y|X)]
\label{eq:maxent} 
\\ \notag
&\text{ such that: }\! \hat{P}(Y|X)\! \in\! \Delta
\text{ and } 
\mathbb{E}_{\Ps(x) \hat{P}(y|x)}[\Phi(X,Y)] = {\bf c}, 
\end{align}
}%
where 
$\Delta$ defines the conditional probability simplex that 
$\hat{P}(y|x)$ 
must reside within, 
$\Phi$ is 
a 
vector-valued 
feature function 
that 
is 
evaluated on input $x$, and ${\bf c} = \mathbb{E}_{\Ps(x,y)}[\Phi(X,Y)]$ 
is a vector of 
the expected feature values that corresponds with the feature function, 
which is approximated using source sample data in practice.
Solving 
for the parametric form of 
$\hat{P}(y|x)$ from this optimization problem 
yields: 
\begin{align}
\hat{P}(y|x)=
{e^{\frac{\Ps(x)}{\Pt(x)} \theta \cdot \Phi(x,y)}}\Big/{Z(x)} , \label{eq:dist}
\end{align}
with a normalization term defined as $Z(x)=\sum_{y' \in \mathcal{Y}} 
\exp\{\frac{\Ps(x)}{\Pt(x)} \theta\cdot \Phi(x,y')\}$. Minimizing the target logarithmic loss,
\begin{align}
\theta=
\argmin_{\theta}\mathbb{E}_{\Pt(x)P(y|x)}[-\log \hat{P}(Y|X)]+\lambda||\theta||^2_2,\label{eq:reg_loss}
\end{align}
provides parameter vector estimates $\theta$.
%
This can be accomplished by approximating the gradient using 
source samples 
rather than approximating the objective  function 
\eqref{eq:reg_loss}.  After plugging \eqref{eq:dist}
into \eqref{eq:reg_loss}, the gradient using source samples is: 
\begin{align}
\mathbb{E}_{\tilde{P}_{\text{src}}(x) \hat{P}(y|x)}[ \Phi(X,Y)]-\tilde{\mathbf{c}} +2\lambda\theta,\label{eq:gradient}
\end{align}
with $\tilde{\bf c} \triangleq \mathbb{E}_{\tilde{P}_{\text{src}}(x)
\tilde{P}(y|x)}[\Phi(X,Y)]$.
Therefore, the RBA approach directly minimizes the expected target logloss rather than approximating it 
using importance weighting 
with 
finite source samples. 

As illustrated by Figure \ref{syn_cmp}, the feature function 
of the RBA predictor $\Phi$ forms the constraints
that prevent high levels of uncertainty (entropy) in the 
target distribution.   
As a result, more extensive sets of feature constraints may be needed 
to appropriately constrain the RBA model to provide more certain predictions 
in portions of the input space where target data is more probable under source distribution, like the intersection
of source and target distribution in Figure \ref{syn_cmp}.  

\subsection{Extended representer theorem for RBA}
Kernel methods are motivated in the RBA approach to provide a more 
sufficiently restrictive set of constraints that forces generalization 
from source data samples to target data. 
However, 
the inability to directly apply empirical risk minimization in the        
RBA approach \eqref{eq:reg_loss} complicates their incorporation 
since kernel method applications often use empirical risk minimization as a 
starting point.  

We extend the representer theorem in 
the 
RBA approach by first
investigating 
the minimizer of the regularized expected target loss. 
Theorem ~\ref{thm:represent} shows that the minimizer of a regularized expected target loss can instead be
represented by a reweighted expectation under the source distribution. This paves the theoretical foundation of applying kernel methods to RBA, which essentially differs from traditional empirical risk minimization based methods and use expected target loss as a starting point.

{
\begin{figure*}[htp] 
\begin{center}
\begin{tabular}{cccc}
(a) Linear & (b) Gaussian & (c) Polynomial-2 & (d) Polynomial-3\\
\includegraphics[height=3.6cm]{figures/linear_crop.png} &
\includegraphics[height=3.6cm]{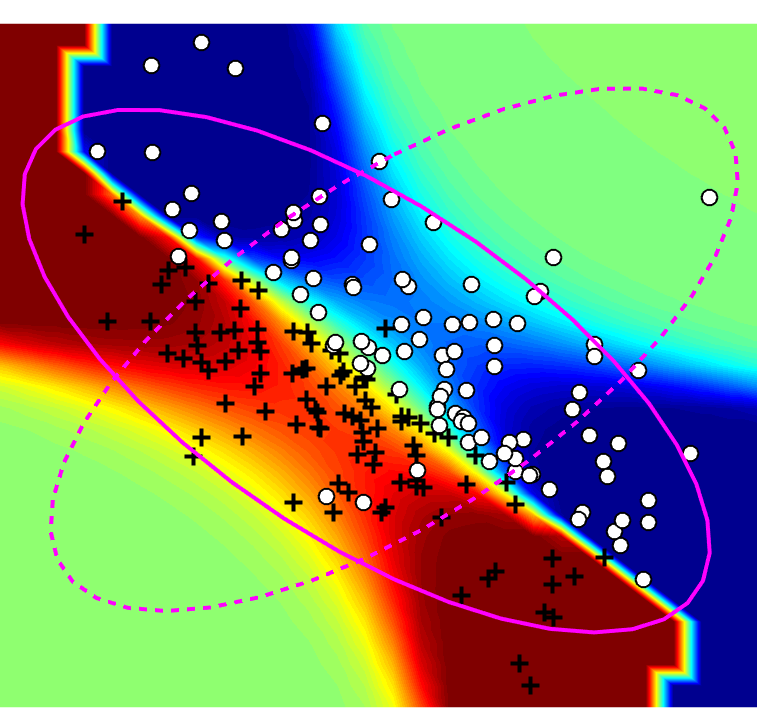} & 
\includegraphics[height=3.6cm]{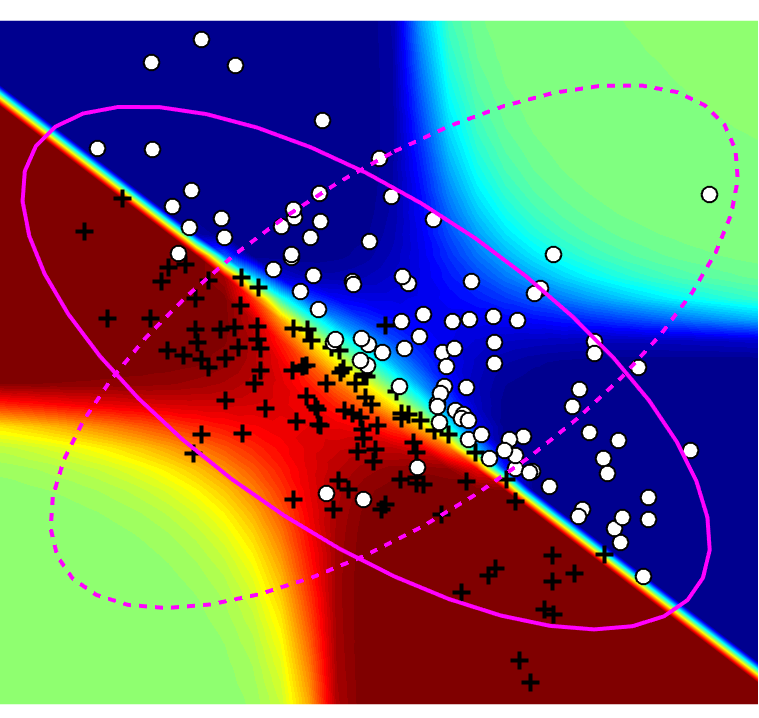}&
\includegraphics[height=3.6cm]{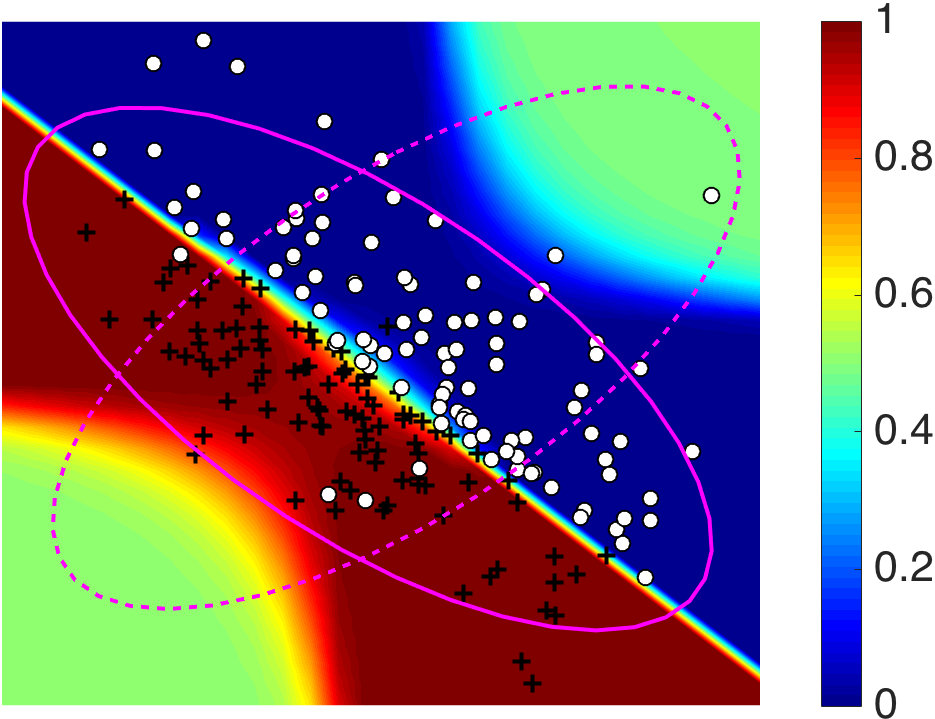} \\
logloss: 0.74&logloss: 0.65&logloss: 0.48  &logloss: 0.41\\
entropy: 0.93 &entropy: 0.86&entropy: 0.67&entropy: 0.45\\
\end{tabular}
\caption{Performance comparison with robust bias aware classifier using 
linear features (a), using Gaussian kernels with bandwidth 0.5 (b), using
polynomial kernels with order 2 (c) and using polynomial kernels
with order 3 (d).
Ellipses show the same source and target data distribution as in Figure~\ref{syn_cmp}.
The intersection of 
source distribution and target distribution is better predicted with 
kernel methods applied. The corresponding logloss and entropy evaluated on 
the target distribution shows that
more certain and informative predictions are produced by kernel RBA.
}
\label{fig:simulation}
\end{center}
\end{figure*}
}
\begin{theorem} \label{thm:represent}
Let $\mathcal{X}$ be the input space and $\mathcal{Y}$ be the output space, $K$ is a positive definite real valued
 kernel on $(\mathcal{X},\mathcal{Y}) \times (\mathcal{X}, \mathcal{Y})$ with corresponding reproducing kernel Hilbert space $H_k$, 
if the training samples $(x_{1}^s, y_{1}^s), \dotsc, (x_n^s, y_n^s) \in \mathcal{X} \times \mathcal{Y}$ are 
drawn from a source distribution $\Ps(x)P(y|x)$ and the testing samples
$(x_{1}^t, y_{1}^t), \dotsc, (x_{m}^t, y_{m}^t) \in \mathcal{X} \times \mathcal{Y}$ are drawn from a
 target distribution $\Pt(x)P(y|x)$, any minimizer $f^{*}$ of \eqref{eq:reg_loss} in $H_k $, defining the conditional label distribution,
\begin{align}
\hat{P}(y|x) = {e^{f^*(x,y)}}\Big/{\sum_{y'} e^{f^*(x,y')}}, 
\label{eq:kernel_dist}
\end{align} 
admits a representation with a form such that
each $f^*(x_{i}^t, y_{i}^t)=$ 
{\small
\begin{align}
\frac{\Ps(x_{i}^t)}{\Pt(x_{i}^t)}\mathbb{E}_{\Ps(x)P(y|x)}\left[\alpha(X,Y)K((x_{i}^t,y_{i}^t),(X,Y))\right],\label{eq:minimizer}
\end{align}}%
where $\alpha(x_i,y_i) \in \mathbb{R}$, for $1 \le i \le m$,
with
\begin{align}
\theta=\mathbb{E}_{\Ps(x)P(y|x)}\left[\alpha(X,Y)\Phi(X,Y)\right]. \label{eq:theta}
\end{align}
\end{theorem}

\begin{proof}
Defining 
$\Phi'(x,y)\triangleq \frac{\Ps(x)}{\Pt(x)}\Phi(x,y)$, the robust bias-aware label distribution can be rewritten as 
$\hat{P}(y|x) = e^{\theta \cdot \Phi'(x,y)} /
Z(x)$, with $Z(x)=\sum_{y' \in \mathcal{Y}} e^{\theta\cdot \Phi'(x,y')}$. The objective function \eqref{eq:reg_loss} is then:
\begin{align}
&\mathbb{E}_{\Pt(x)P(y|x)}[-\log \hat{P}(Y|X)]+\lambda||\theta||^2_2\\
&=\mathbb{E}_{\Pt(x)P(y|x)}[-f(X,Y)+\log Z(X)]+\lambda ||\theta||^2_2\notag,
\end{align}
where $f(x,y)=\langle \Phi'(x,y), \theta\rangle$ is the function that we aim to find that minimizes this regularized expected loss.
Let $K'$ be a positive definite real valued kernel on $H_k'$, 
according to the generalized representer theorem  \cite{de2004some} in this expected risk case, the minimizer $f^*$
takes 
the form:
\begin{align}
f^*(x_{i}^t,y_{i}^t)=\mathbb{E}_{\Pt(x)P(y|x)}[\alpha(X,Y)K'((x_{i}^t,y_{i}^t),(X,Y))],\notag
\end{align}
where $K'((x_{i}^t,y_{i}^t),(x,y))=\langle \Phi'(x_{i}^t,y_{i}^t) ,\Phi'(x,y)\rangle$. 
Since the target label is not 
available in training, the minimizer cannot be represented directly by target data. 
Instead, we 
represent it using 
source data, which, 
for each $1 \le i \le m$, 
is: 
{\small
\begin{align}
 &f^*(x_{i}^t,y_{i}^t) =  \notag\\
&\mathbb{E}_{\Pt(x)P(y|x)}\bigg[\frac{\Ps(x_{i}^t)}{\Pt(x_{i}^t)}\frac{\Ps(X)}{\Pt(X)}\alpha(X,Y) 
K((x_{i}^t,y_{i}^t),(X,Y))\bigg]\notag\\
&=\frac{\Ps(x_{i}^t)}{\Pt(x_{i}^t)}\mathbb{E}_{\Ps(x)P(y|x)}\left[\alpha(X,Y)K((x_{i}^t,y_{i}^t),(X,Y))\right].\notag
\end{align}
}
Given $f(x,y)=\langle \Phi'(x,y), \theta\rangle=\frac{\Ps(x)}{\Pt(y)}\theta\cdot \Phi(x,y)$, we obtain $\theta=\mathbb{E}_{\Ps(x)P(y|x)}\left[\alpha(X,Y)\Phi(X,Y)\right]$.

\end{proof}

\subsection{Kernel RBA parameter estimation}
As in the non-kernelized RBA model, the objective function  
\eqref{eq:reg_loss} is defined in terms of the labeled target  
distribution data, which is unavailable. 
However, the parametric model's form \eqref{eq:kernel_dist} bypasses this difficulty 
when employing the kernelized minimizer 
\eqref{eq:minimizer}.  
In order to estimate the parameters $\{\alpha(x,y)\}$, 
we derive the gradient of the kernel RBA predictor.

\begin{corollary}[of Theorem \ref{thm:represent}]
The gradient (with respect to kernelized parameters $\alpha$) of the
regularized expected loss is obtained by approximating 
kernel evaluations under the source distribution with source sample
kernel evaluations.
\label{cor:gradient}
\end{corollary}
\begin{proof}
Plugging \eqref{eq:theta} into \eqref{eq:reg_loss}, we obtain the form of the
objective function represented by kernels and take derivatives 
with respect to $\alpha(x',y')$:

{\small
\begin{align}
&\frac{\partial}{\partial\alpha(x',y')}\left(\mathbb{E}_{\Pt(x)P(y|x)}[-\log P_{\theta}(Y|X)]+\lambda||\theta||^2_2\right)\notag\\
&=-\mathbb{E}_{\Ps(x)P(y|x)}[K((x', y'),(X, Y))] \notag\\
&+\mathbb{E}_{\Ps(x)\hat{P}(y|x)}[K((x', y'),(X, Y))]\notag\\
&+\lambda\mathbb{E}_{\Ps(x'')P(y''|x'')}[\alpha(X'',Y'')K((x',y'),(X'',Y''))]\notag\\
&\approx-\mathbb{E}_{\tilde{P}_{\text{src}}(x)\tilde{P}(y|x)}\left[K((x',y'),(X,Y))\right]\notag\\
 &+\mathbb{E}_{\tilde{P}_{\text{src}}(x)\hat{P}(y|x)}\left[K((x',y'),(X,Y))\right]\notag\\
&+\lambda\mathbb{E}_{\tilde{P}_{\text{src}}(x'')\tilde{P}(y''|x'')}[\alpha(X'',Y'')K(x',y'),(X'',Y'')].\notag
\end{align}
}
\end{proof}

Corollary ~\ref{cor:gradient} indicates that the computation of the gradient only requires 
source samples. 
This requires an approximation of the source distribution's expected kernel %
evaluations with the empirical evaluations of the sample mean. The reason for the approximation is rooted in the idea of minimizing the exact expected target loss directly in kernel RBA. Consequently, we need to use the empirical gradient to approximate the true gradient. However, the error
can be controlled using standard finite sample bounds, like Hoeffding bounds, so that the corresponding error in the objective is also bounded. On the contrary, importance weighted empirical risk minimization (ERM) methods do not approximate the gradient, but approximate the training objective from the beginning as in \eqref{eq:reweight}, which is essentially different from our method.

{
\begin{figure*}[htp] 
\begin{center}
\begin{tabular}{cccc}
(a) Linear-100  & (b)  Gaussian-200 & (c)  Gaussian-300 & (d)  Gaussian-400 \\
\includegraphics[height=3.6cm]{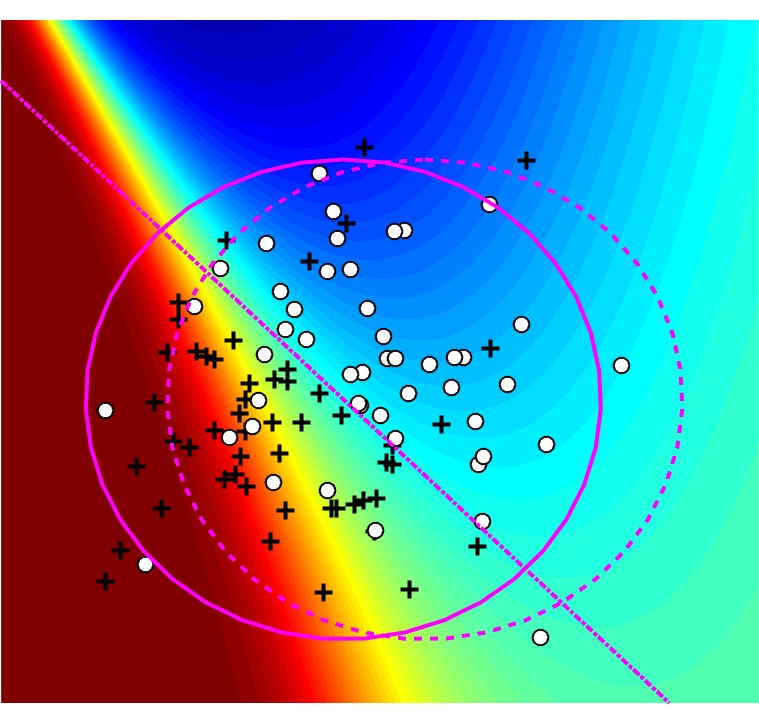} &
\includegraphics[height=3.6cm]{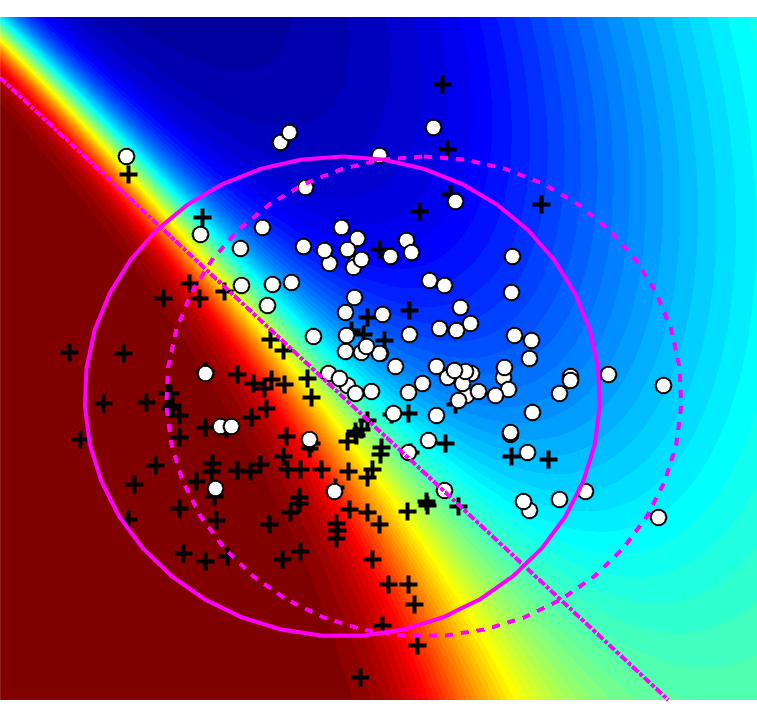}&
\includegraphics[height=3.6cm]{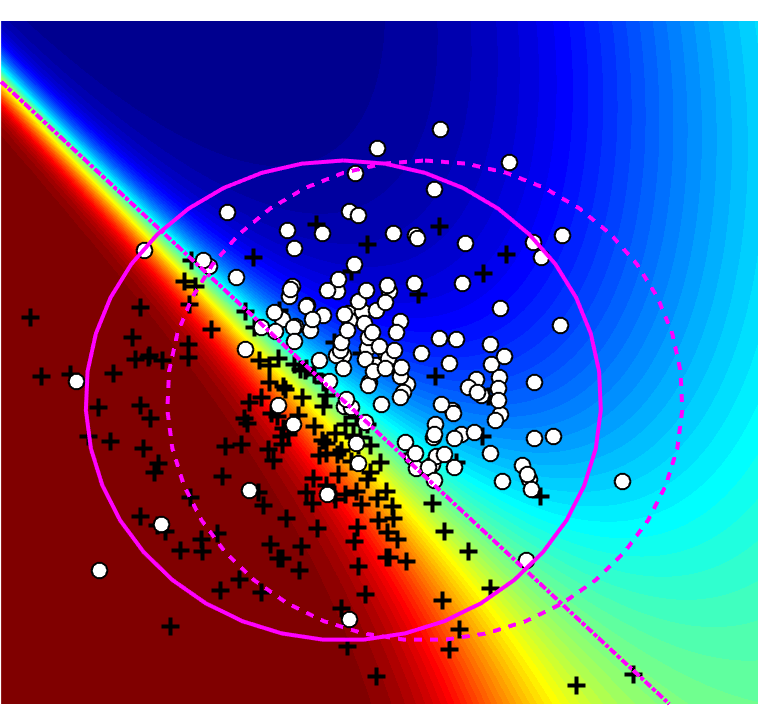} &
\includegraphics[height=3.6cm]{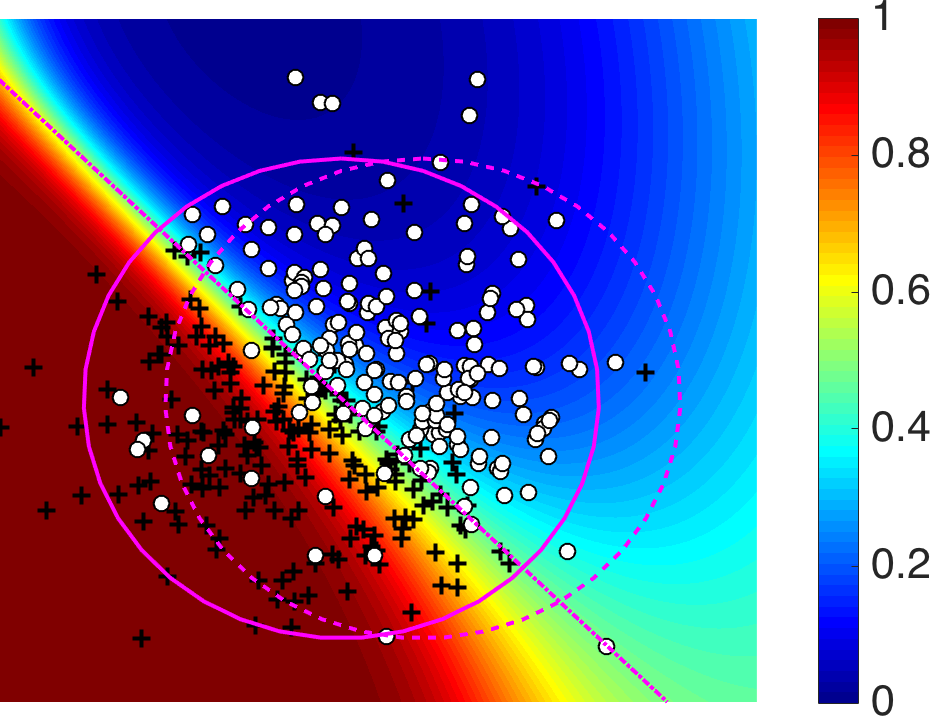}\\
accuracy: 0.760&accuracy: 0.774 &accuracy: 0.789& accuracy: 0.815\\
\end{tabular}
\caption{Convergence of decision boundary in robust bias aware classifier using 
linear features on 100 samples (a), using Gaussian kernels on 200 samples (b), on 300 samples (c) and on 400 samples (d), with 20\% noise in each example.
Ellipses show source and target data distribution that closely overlap. The tiled line shows the true decision boundary.
With an increasing number of samples and universal kernels, the true decision boundary is recovered with accuracy gradually converging to optimal.
}
\label{fig:convergence}
\end{center}
\end{figure*}
}

%

\subsection{Understanding Kernel RBA}
In order to illustrate the effectiveness of kernel RBA, we consider the same datasets from Figure~\ref{syn_cmp} and compare
linear RBA and kernel RBA with different kernel types and parameters in Figure~\ref{fig:simulation}. Even though kernel methods are
usually regarded as a way to introduce non-linearity, its main effect in kernel RBA is the expansion of the constraint space for
the adversarial player $\check{P}(Y|X)$ in the two player game in (\ref{eq:game}). As in Figure~\ref{fig:simulation}, 
kernel RBA achieves better (smaller target logarithmic loss) and more informative (smaller target prediction entropy) predictions in
the intersection of source and target distribution, while the true decision boundary is a linear one. Note that here the Gaussian kernel has a large bandwidth to obtain a more linear decision boundary for better visualization.
Moreover, the difference
between target entropy and logarithmic loss gradually gets smaller in the last three figures. This corresponds with the property of RBA that
target logarithmic loss is always upper bounded by the target entropy (with high probability), as proven for a general case in previous literature  \cite{liu2015shift}. Therefore, when a larger number of constraints are imposed, 
i.e., kernel methods are applied, it forms a more restrictive constraint set for $\check{P}(Y|X)$ so that target entropy will bound target loss more and more tightly.

Note that the choice of kernel method and kernel parameters depends on the specific learning problem because we also need to account for overfitting issues in practice. The amount of bias also plays a role in how more source constraints brought by kernel methods help improve over RBA method. Specifically, the larger the bias is, the more RBA will suffer from insufficient constraints from source sample data, which results in larger entropy in target predictions.

\subsection{Consistency Analysis}
We now analyze some theoretical properties of the kernel RBA method. As stated before, kernel RBA directly minimizes the regularized expected target loss.
We start with defining this expected target loss explicitly, parameterized by learned $\theta$, at a specific data point $(x, y)$ as: 
$
L_{RBA} (x, y) 
= 
\gamma(\theta, x, y) - \log Z$,
where $\gamma(\theta, x, y) = \frac{P_{\text{src}}(x)}{P_{\text{trg}}(x)}\theta \Phi(x, y)$ and $\log Z$ is the normalization term. 
\begin{theorem} \label{thm:consistent}
Let $k$ be an bounded universal kernel, and regularization $\lambda$ tending to zero slower than $1/m$ for the kernel RBA method, with $\hat{\theta}$ as the parameter in the resulting predictor, then $\mathbb{E}_{P_{\text{trg}(x, y)}}[L_{RBA}(\hat{\theta}, x, y)] - \mathbb{E}_{P_{\text{trg}}(x, y)}[L_{RBA}( \theta^*,x, y)] \xrightarrow{a.s.} 0$. 
\end{theorem}
\begin{proof}
$L_{RBA}$ is a Lipschitz loss because it follows the basic form of logistic loss except $\gamma(\theta, x, y)$ consists of one more
component: 
the density ratio.  Given Theorem ~\ref{thm:represent}, the minimizer of expected target $L_{RBA}$ can be represented using source samples. It implies that kernel RBA
is consistent w.r.t $\mathbb{E}_{P_{\text{trg}(x, y)}}[L_{RBA}(\theta, x, y)]$ when equipped with a universal kernel \cite{micchelli2006universal} in source data, assuming $\frac{\Ps(x)}{\Pt(x)}$ is accurate, according
to consistency properties for Lipschitz loss  \cite{steinwart2005consistency}. 
\end{proof}
Next, we explore whether the optimal expected $L_{RBA}$ on the target distribution $\mathbb{E}_{P_{\text{trg}}(x, y)}[L_{RBA}( \theta^*,x, y)]$ indicates the optimal 0-1 loss on the target distribution\footnote{We assume the density 
ratio $\Ps(x)/\Pt(x)$ is accurately estimated in this case and leave the analysis for the case when it is
approximate to future work.}. 
\begin{corollary}[of Theorem \ref{thm:consistent}]
For any pair of distributions that $P_{\text{src}}(x) > 0$, $P_{\text{trg}}(x) > 0$ and $P_{\text{src}}(y|x) = P_{\text{trg}}(y|x)$, if $\hat{\eta}(x)$ is the kernel RBA 
predictor satisfying all the conditions in Theorem ~\ref{thm:consistent}, then $\mathbb{E}_{P_{\text{trg}}(x, y)}[L_{0-1}(\hat{\eta}(x), y)] - \mathbb{E}_{P_{\text{trg}}(x, y)}[L_{0-1}( \eta^*(x), y)] \xrightarrow{a.s.} 0$.
\end{corollary}
\begin{proof}
$L_{RBA}$ is a proper composite loss in both the binary \cite{reid2010composite} and multi-class cases \cite{vernet2011composite}, which means it satisfies $L_{RBA}(\eta, \hat{\eta}) - L_{RBA}(\eta, \eta) \ge \frac{C}{2}(\hat{\eta} - \eta)^2 $ for any $\eta, \hat{\eta} \in [0, 1]$, where $\eta$ is the Bayes conditional label probability, $\hat{\eta}$ is the estimated label probability function $\eta(\hat{\theta}, x)$ from RBA (\ref{eq:dist}) and $C > 0$ is a constant.  We then have target expected 0-1 regret be bounded by the expected $L_{RBA}$ regret:
\begin{align}
&\mathbb{E}_{P_{\text{trg}}(x, y)}[L_{0-1}(\hat{h}(x), y)] - \mathbb{E}_{P_{\text{trg}}(x, y)}[L_{0-1}(h^*(x), y)]\notag \\
& \le 2\sqrt{\mathbb{E}_{P_{\text{trg}}(x, y)}[\hat{\eta}(x) - \eta^*(x)]^2}\notag\\
&  \le 2 \sqrt{\frac{2}{C} \mathbb{E}_{P_{\text{trg}}(x, y)}[L_{RBA}(\hat{\eta}(x)) - L_{RBA}(\eta^*(x))]},  \notag
\end{align}
where $h$ is a predictor function that maps conditional label probability $\eta(x)$ to label. Here the first inequality is due to property of plug-in classifiers and Jensen's inequality and the second inequality directly comes from the definition of proper loss. 
Therefore, according to Theorem ~\ref{thm:consistent}, kernel RBA is consistent w.r.t $L_{RBA}$, and we then conclude that $\mathbb{E}_{P_{\text{trg}}(x, y)}[L_{0-1}(\hat{\eta}(x), y)] - \mathbb{E}_{P_{\text{trg}}(x, y)}[L_{0-1}( \eta^*(x), y)] \xrightarrow{a.s.} 0$.
\end{proof}
Note that employing a universal kernel is a sufficient condition for consistency to hold. 
Therefore, kernel methods not only provide a larger number of features without increasing computational burdens, but also facilitate the theoretical property to hold
for kernel RBA. 

We demonstrate how the true decision boundary in the target distribution is recovered with an increasing number of samples when source and target distribution are fairly close in Figure \ref{fig:convergence}. As shown in the first figure, the decision boundary in the linear case is tilted due to the noise. Equipped with more samples and a universal kernel (Gaussian kernel), the decision boundary is shifted to align with the true one. At the same time, the accuracy on target data gets better and better, roughly converging to the optimal. This property of kernel RBA corresponds to Corollary \ref{thm:consistent} that the 0-1 loss of kernel RBA should converge to the optimal 0-1 loss in the limit.

As a comparison, we show the plots of logloss and accuracy of Kernel IW (solid line) and Kernel Robust (dashed line) methods after 20 repeated experiments using increasing number of samples in Figure \ref{fig:Loss_converge}. The dataset is similar with the example in Figure \ref{fig:convergence} with 10\% noise and source and target distribution closely overlapped. The kernel used here is Gaussian kernel. As shown in the error bars, even though the importance weighted loss converges to the target loss in the limit in theory, it suffers from larger variance and sensitivity to noise in reality when there is only limited number of samples. The reason is that it can be dominated by data with large $\Pt(x)/\Ps(x)$ weights, like points with `+' labels in the right-upper corner in Figure \ref{fig:convergence}. Those noise points will push the decision boundary to the left-bottom direction in order to suffer less logloss. On the other hand, Kernel Robust is more robust to noise and keeps reducing the variance and improving the mean logloss and accuracy. This is not only due to the inherently more modest predictions that robust methods produce on biased target distribution, but also due to the consistency property it enjoys as stated in Theorem \ref{thm:consistent} and Corollary \ref{thm:consistent}.
Even though the number of samples is still small and limited here, the source and target distribution is close enough to reflect the convergence tendency with the increasing of source samples.

{
\begin{figure}[htp] 
\begin{center}
\begin{tabular}{cc}
(a) Logloss  & (b)  Accuracy \\
\includegraphics[height=3cm]{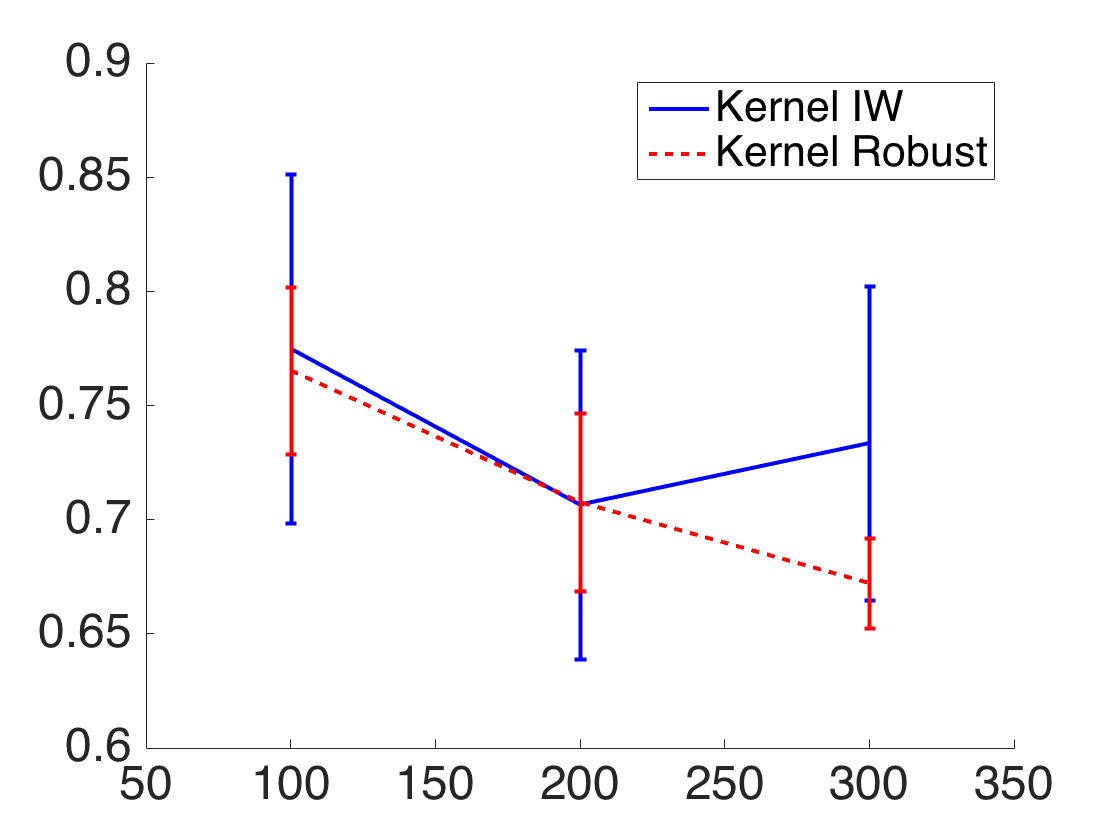} &
\includegraphics[height=3cm]{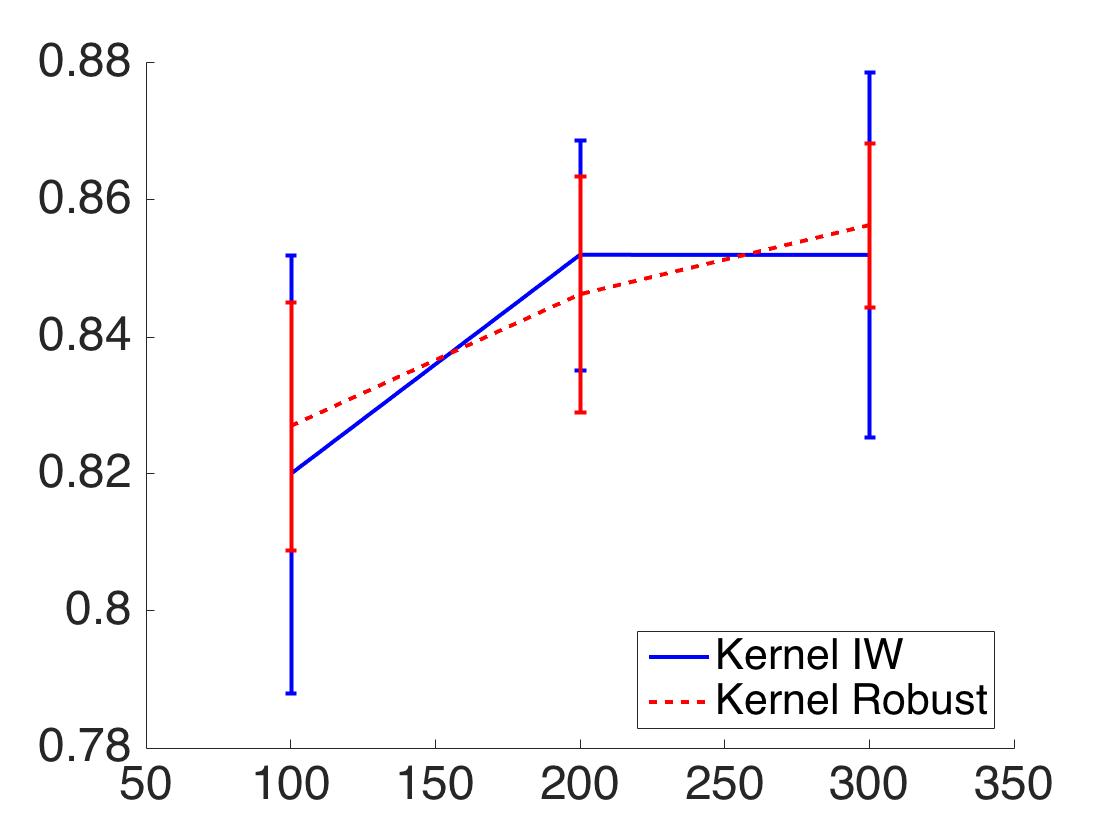}\\
\end{tabular}
\caption{Logloss and accuracy plots as sample size increases from 100 to 300 in kernel IW and kernel Robust methods, with Gaussian kernel, for datasets similar in Figure \ref{fig:convergence}. The error bar shows the 95\% confidence interval of the sampling distribution after 20 repeated experiments. IW methods suffer from large variance as robust methods gradually reduce variance and improves on logloss and accuracy more consistently.
}
\label{fig:Loss_converge}
\end{center}
\end{figure}
}

\section{Experiments}
In this section, we demonstrate the advantages of our kernel RBA approach 
on datasets that are either synthetically  
biased via sampling 
or naturally biased by a differing characteristic or noise.
We chose three datasets from the UCI repository  \cite{UCI,dataset_vehicle} for synthetically biased experiments, based on the criteria that each contains approximately 
1,000 or more examples and has minimal missing values. They are \verb|Vehicle|, \verb|Segment| and \verb|Sat|. For each dataset, we
 synthetically  
generate 20
separate 
experiments by taking 
200 source samples and 200 target data samples 
from it, generally following the sampling procedure described in \citeauthor{huang2006correcting} (\citeyear{huang2006correcting}), which 
we summarize as: 
\begin{enumerate}[topsep=0pt,itemsep=-1ex,partopsep=1ex,parsep=1ex]
\item Separate the data into source and target portion according to mean of a variable;
\item Randomly sample the target portion as the target dataset;
\item In the source portion, calculate the sample mean $\mu$ and sample covariance $\sigma$, then 
sample in proportion to weights generated from a multivariate Gaussian with $\mu' = \mu/5 $ and $\sigma' = \sigma/5$ as the source dataset. If the dimension
is too large to sample any points, perform PCA first and use the first several principle components to obtain the weights.
\end{enumerate}

We also investigate 
three naturally biased 
covariate shift datasets. One of them is \verb|Abalone|,
in which we use the sex variable (male, female, and infant) to create bias. Specifically, we use infant as source samples and the rest as target samples. 
Note that we use the simplified 3-category classification problem of the \verb|Abalone| dataset as described in Clark 
\emph{et al.} \cite{clark1996quantitative} and also sample 200 data points respectively for the source and target datasets.
We chose this data because the sex variable makes source-target separation easier and reasonable, 
and allows the covariate shift assumption to generally hold. In addition, we evaluate our methods on the MNIST dataset \cite{lecun1998gradient}, which we reduc to binary predictive tasks of differentiating `3' versus `8' and `7' versus `9'. We add a biased Gaussian noise with mean 0.2 and standard deviation 0.5 to the testing data to form the covariate shift, i.e. noise $z \sim N(0.2, 0.5)$. We randomly sample 2000 training and testing samples and repeat the experiments 20 times. Shown in Figure \ref{fig:mnist} is the comparison between one batch of training samples and testing samples. 

{
\begin{figure}[htp] 
\begin{center}
\begin{tabular}{cc}
(a) Training Samples  & (b) Testing Samples  \\
\includegraphics[height=2cm]{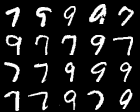} &
\includegraphics[height=2cm]{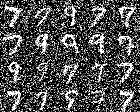}\\
\end{tabular}
\caption{Binarized MNIST data with noise added to the testing set to form covariate shift.
}
\label{fig:mnist}
\end{center}
\end{figure}
}

{
\begin{table*}[t]
\begin{center}
\caption{Average Target Logloss Comparison} \label{tab:result_average_1}
{\small
\begin{tabular} { |K{1.8cm}| K{2cm}| K{1.8cm}| K{1.8cm} || K{1.5cm} | K{1.5cm}| K{1.5cm}|}
\hline 
{\bf Dataset} & {\bf Kernel Robust} &{\bf Kernel LR} 
& {\bf Kernel IW}&{\bf Robust}&{\bf LR}&{\bf IW}\\
\hline
\verb|Vehicle| & {\bf 1.92} & 16.41 & 87.69& 1.94  & 8.15 & 4.94 \\
\hline
\verb|Segment| & {\bf 2.53} & 9.62& 83.75 & 2.55& 4.37& 4.01\\
\hline 
\verb|Sat| & {\bf 2.44} & 205.27&  111.57& 2.57 & 13.27&8.95 \\
\hhline{|=|=|=|=#=|=|=|}
\verb|Abalone|& {\bf 1.58} &  8.52 &  6.91  &{\bf 1.59}& 8.73& 2.09 \\
\hline
\verb|MNIST-7v9|& {\bf 0.42} &  0.44 &  0.49  &0.55& 0.80& 0.59\\
\hline
\verb|MNIST-3v8 |& {\bf 0.39} &  0.46 &  0.41  &0.48& 0.84& 0.60 \\
\hline
\end{tabular}
}
\end{center}
\end{table*}
}

\subsection{Methods}

We evaluate our approach and five other methods:\\ 
%
{\bf Kernel robust bias aware classifier (Kernel Robust)} adversarially
minimizes the target distribution logloss using kernel methods, trained using direct gradient calculations as in Corollary  \ref{cor:gradient}.\\
{\bf Kernel logistic regression (Kernel LR)} 
ignores the covariate shift and
maximizes the source data conditional likelihood, $\max_{\theta} \mathbb{E}_{\Ps(x)P{(y|x)}} \left[\log P_{\theta} (Y|X) \right] - \lambda \| \theta \|^2_2  $, where $ \hat{P}_{\theta} (y|x) = \frac{\exp(\theta \cdot \Phi(x,y))}{\sum_{y' \in \mathcal{Y}} \exp(\theta \cdot \Phi(x,y'))} $ and $\lambda$ is the regularization constant.   
\\ 
{\bf Kernel importance weighting method (Kernel IW)} 
maximizes the conditional target data likelihood as estimated using importance
weighting 
with the density ratio, {\small$\max_{\theta} \mathbb{E}_{\Ps(x)P(y|x)} \left[  \frac{\Pt(x)}{\Ps(x)} \left(\log P_{\theta} (Y|X)  \right)\right] - \lambda \| \theta \|^2_2$}.\\
{\bf Linear robust bias aware prediction (Robust)} adversarially minimizes the target distribution logloss without utilizing kernelization , i.e. only first order features are used, trained using direct gradient calculations (\ref{eq:gradient}).\\
{\bf Linear logistic regression (LR)} utilizes only first order features in the source conditional log likelihood maximization.\\
{\bf Linear importance weighting method (IW)} uses first order features only to maximize reweighted source likelihood.

\subsection{Model Selection}
For each kernelized method, 
we employ a polynomial kernel with order 2.
We choose regularization parameter $\lambda$ 
by 5-fold cross validation, or importance weighted cross validation (IWCV) from 
$\lambda \in [2^{-16}, 2^{-12}, 2^{-8}, 2^{-4}, 1]$.
We apply traditional cross validation on Kernel LR and LR, and apply IWCV on both importance weighting methods and robust methods.
Note that the traditional cross validation process is not correct anymore in the covariate shift setting, because under the covariate shift assumption, the source marginal data distribution of $P(x)$ is
different from the target distribution  \cite{sugiyama2007covariate}. 
Though IWCV was originally designed for the importance weighting methods, 
it is proven to be unbiased for any loss function.  We apply it to perform model tuning for our robust methods, even though the error estimate variance could 
be large.


\subsection{Logistic regression as density estimation}
We use a discriminative density estimation method that leverages the logistic regression classifier 
for estimating the density ratios. 
According to Bayes rule:
$
\frac{\Ps(x)}{\Pt(x)} = \frac{P(x | \text{``source"})}{P(x | \text{``target"})}
= \frac{P(\text{``source"}|x)}{P(\text{``target"}|x)} \frac{P(\text{``target"})}{P(\text{``source"})}, 
$
where the second ratio $P(\text{``target"}) / P(\text{``source"})$ is computed as the ratio of the number of target and source examples, and the first one is obtained by training a classifier with source data labeled as one class and target data as another class. Similar ideas also appears in recent literature \cite{lopez2016revisiting}.
The resulting density ratio of this method is also closely controlled by the amount of 
regularization. We also choose the regularization weight by cross validation. 

\subsection{Performance Evaluation}
We compare average logloss, 
$\mathbb{E}_{\tilde{P}_{\text{trg}}(x) \tilde{P}(y|x)}[-\log_2 \hat{P}(Y|X)]$,
for each method in Table~\ref{tab:result_average_1}.
We perform a paired t-test among each pair of methods.
We indicate the methods that have the best performance in bold, along with
methods that are statistically indistinguishable from the best (paired
t-test with $0.05$ significance level).
As shown from the table, the average logloss of the Kernel Robust method is significantly better or not significantly worse than all of the alternatives in all of the datasets. Moreover, we observe the following:

First, logloss of Kernel Robust and Robust is bounded by the uniform distribution baselines, while LR and IW methods can be arbitrary worse when the bias is large, like in \verb|Vehicle|. This aligns with the properties of robust methods because when the bias is large, the density ratio becomes small and results in uniform predictions. This indicates that robust methods should be preferred if robustness or safety is a concern when the amount of covariate shift is large.

Secondly, Kernel Robust consistently improves the performance from Robust while kernelization may harm LR and IW methods, like in \verb|Sat|. The reason is when the implicit assumption that (reweighted) source features can be generalize to target distribution in LR and IW does not hold anymore, incorporating larger dimensions of features could make predictions worse. For Kernel Robust and Robust, even though overfitting could still be a concern, the density ratio could adjust the certainty of the prediction and function like a regularizer based on the data's density in training and testing distribution, so that they suffer less from overfitting.

Finally, we find that Kernel Robust improvement over Robust is related to how far the source input distributions is from the target input distribution. The natural bias in \verb|Abalone| comes from one feature variable and could be smaller than the bias in synthetic data. This could be why the improvement of logloss in Abalone is smaller than other datasets.

\section{Conclusion}
Providing meaningful and robust predictions
under covariate shift is challenging.
Kernel methods 
are one avenue for considering large or infinite feature spaces without 
incurring a proportionate computational burden.  
We investigated the underlying theoretical foundations for applying kernel methods to RBA by extending
the generalized representer theorem, which makes it possible to represent the minimizer of the regularized 
expected loss with reweighted kernel expectations under the source distribution, and therefore minimize
the objective using gradient calculations that only depend on source samples. In addition, we presented the implication of kernel RBA in providing more restrictive feature matching constraints 
and tighter entropy bounds for target loss, and demonstrated that 
kernel RBA is both consistent w.r.t its own expected target loss and 0-1 loss. 
We experimentally validated the advantages of kernelized
RBA with synthetically subsampled benchmark data and naturally biased data.  


\bibliographystyle{aaai}
\bibliography{biblio}

\newpage
\appendix
\onecolumn
\section{SUPPLEMENTARY MATERIALS}
\subsection{Dataset Details}\label{sup:details}
We show the more detailed information about the datasets we used in the experiment in the following tables.
We expect the method to also work for higher dimensional dataset when equipped with accurate density ratio estimation. Since the development and analysis of this paper focus more on the Kernel RBA method itself and not on density estimation, we believe smaller datasets are more suitable for the evaluation. We leave the problem of being robust to possibly inaccurate density ratios in higher dimension to future work.
\begin{table}[h]
\begin{center}
\caption{
 Biased Datasets } \label{tab:datasets_info_1}
{\small
\hspace*{-0.2cm}\begin{tabular}{|c|c|c|c|c|c|}
\hline 
{\bf Dataset} & {\bf Features}  & {\bf Examples} 
& {\bf Classes} \\
\hline
\verb|Vehicle| & 18 & 846 & 4\\
\hline
\verb|Segment| & 19 & 2310 & 7\\
\hline
\verb|Sat| & 36 & 6435 & 7 \\
\hline
\verb|Abalone| & 7 & 4177& 3\\
\hline
\verb|MNIST-3v8| & 784 & 5885 & 2\\
\hline
\verb|MNIST-7v9| & 784 & 5959 & 2\\
\hline
\end{tabular}
}
\end{center}
\end{table}

\subsection{Accuracy analysis} \label{sup:accuracy}
We
investigate 
the accuracy (the complement of the misclassification error)
of the predictions provided by each of the six approaches on both
synthetically biased datasets and naturally biased datasets (in Table~\ref{tab:result_accuracy_1}), 
where the significant best performance in paired t-test are demonstrated in bold numbers. The significance level here is 0.05. 
Despite the discrepancy between the logarithmic loss and the
misclassification error, the
Kernel Robust approach provides statistically better performance than other alternative methods, except on the \verb|Abalone| dataset.
The logarithmic loss is an upper bound of the 0-1 loss. 
However, 
the bound can be somewhat loose, so a lower log loss does not necessarily 
indicate 
a smaller classification 
error rate. 
This is a natural outcome of using logarithmic loss for convenience of optimization. 
Since logloss is the natural loss measure for probabilistic prediction and is being optimized by all methods (and not accuracy), we validate our method by comparing to other methods using it.  
Accuracy and logloss do not correlated perfectly, so it is unsurprising that this small difference exists on a measure not being directly optimized. 

{
\begin{table*}[htp]
\begin{center}
\caption{Average Accuracy Comparison} \label{tab:result_accuracy_1}
{\small
\begin{tabular}{|K{1.8cm}|c|K{1.8cm}|K{1.8cm}||K{1.8cm} |K{1.8cm} | K{1.8cm}|}
\hline 
{\bf Dataset} & {\bf Kernel Robust}  & {\bf Kernel LR} 
& {\bf Kernel IW}&{\bf Robust}&{\bf LR} &{\bf IW}\\
\hline
\verb|Vehicle| & {\bf 38\%}& 37\%& 33\% & 36\% & 36\% & 28\%\\
\hline
\verb|Segment| & {\bf 71\%}& 70\%& 37\% & 67\% & 68\% & 36\%\\
\hline
\verb|Sat| & {\bf 33\%}& 30\% & 28\% & 10\% & 10\% & 16\%\\
\hhline{|=|=|=|=#=|=|=|}
\verb|Abalone|& 46\% & 43\%& 42\%& {\bf 48\%} & 47\% & 39\%\\
\hline
\verb|MNIST-3v8|& {\bf 88\%} & 86\%& 86\%& 87\% & 75\% & 85\%\\
\hline
\verb|MNIST-7v9|& {\bf 87\%} & 85\%& 86\%& 86\% & 71\% & 83\%\\
\hline
\end{tabular}
}
\end{center}
\end{table*}
}

\end{document}